\documentclass[conference]{IEEEtran}
\IEEEoverridecommandlockouts
\usepackage{cite}
\usepackage{amsmath,amssymb,amsfonts,amsthm}
\usepackage{algorithmic}
\usepackage[dvipdfmx]{graphicx}
\usepackage{textcomp}
\usepackage{xcolor}
\def\BibTeX{{\mathbb{R}m B\kern-.05em{\sc i\kern-.025em b}\kern-.08em
    T\kern-.1667em\lower.7ex\hbox{E}\kern-.125emX}}

\usepackage{bm}
\usepackage{url}
\usepackage{booktabs}  
\usepackage{comment}
\usepackage{adjustbox}
\theoremstyle{definition}
\newtheorem{theorem}{Theorem}

\begin{document}

\title{Binary Classification from Coupled Pairwise Labels\\
\thanks{This work was supported in part by the Japan Society for the Promotion of Science through Grants-in-Aid for Scientific Research (C) (26K14831).}
}

\author{
\IEEEauthorblockN{
Tomoya Tate\IEEEauthorrefmark{1},
Kosuke Sugiyama\IEEEauthorrefmark{1},
and Masato Uchida\IEEEauthorrefmark{1}
}
\IEEEauthorblockA{\IEEEauthorrefmark{1}Waseda University, Tokyo, Japan\\
Email: tomot203@toki.waseda.jp, kohsuke0322@asagi.waseda.jp, m.uchida@waseda.jp}
}

\maketitle

\begin{abstract}
Even when it is difficult to assign absolute class labels to individual instances, relational information may still be available, such as whether two instances belong to the same class or which instance is more likely to belong to the positive class. 
In this study, we refer to these two types of information as Similarity/Dissimilarity (SD) labels and Pairwise Comparison (Pcomp) labels, respectively, and consider binary classification that uses both types of relational information from the same instance pairs. 
SD learning uses the distinction between similar and dissimilar pairs but does not use the ordering within each pair, whereas Pcomp learning uses the ordering within each pair but does not distinguish between similar and dissimilar pairs. 
We therefore propose SD-Pcomp learning, whose objective function simultaneously preserves the structures of both SD learning and Pcomp learning. 
The proposed objective function admits two decompositions: one consists of an SD estimator plus a term that represents ordering information from Pcomp labels, and the other consists of a Pcomp estimator plus a term that represents pair-type information from SD labels. 
These decompositions clarify how the complementary information provided by SD and Pcomp labels is integrated into the proposed objective function. 
Experiments on eight datasets compare the proposed method with SD learning, Pcomp learning, and a method that takes a convex combination of their objective functions. 
We evaluate the effect of using both types of relational information on classification performance in terms of classification accuracy and AUC.
\end{abstract}

\begin{IEEEkeywords}
weakly supervised learning, binary classification, similarity/dissimilarity labels, pairwise comparison labels
\end{IEEEkeywords}

\section{Introduction}

Supervised binary classification assumes that a positive or negative class label is observed for each training instance. 
However, when evaluation is subjective or the class boundary is ambiguous, it can be difficult to establish a criterion for classifying each instance independently. 
In such cases, relational information between two instances can be used as supervision instead of absolute class labels for individual instances. 
For example, it may be easier to determine whether two instances belong to the same class or which of the two is more likely to belong to the positive class than to classify each instance independently. 
The former has been formulated as a Similarity/Dissimilarity (SD) label~\cite{shimada2021}, whereas the latter has been formulated as a Pairwise Comparison (Pcomp) label~\cite{feng2021}. 
Learning from such relational information between instances is regarded as a form of weakly supervised learning~\cite{sugiyama2022}.

SD and Pcomp labels provide complementary relational information. 
SD learning distinguishes between similar and dissimilar pairs but does not use the ordering within each pair, whereas Pcomp learning uses the within-pair ordering but does not distinguish between similar and dissimilar pairs. 
This complementarity motivates us to consider SD-Pcomp classification, in which both an SD label and a Pcomp label are assigned to the same instance pair. 
A straightforward approach to integrate these two types of information is to take a convex combination of the objective functions for SD learning and Pcomp learning. 
However, this approach first constructs the two objective functions independently and then combines them. 
Therefore, it does not directly reflect the correspondence between the two types of relational information assigned to the same instance pair. 
We instead construct a single objective function so that it recovers the structure of SD learning when it is aggregated with respect to the similar/dissimilar relation, and it recovers the structure of Pcomp learning when it is aggregated with respect to the comparison direction within each pair. 
This construction yields an objective function that consistently reflects the two types of relational information contained in the same instance pair. 
The resulting objective function admits two decompositions. One decomposition consists of the SD estimator and an additional term that represents ordering information from the Pcomp labels. 
The other consists of the Pcomp estimator and an additional term that represents pair-type information from the SD labels. 
The former additional term encourages score orderings that are consistent with the Pcomp labels, whereas the latter reflects the distinction between similar and dissimilar pairs in the objective function. 
These two decompositions clarify how the two types of relational information complement each other in the proposed objective function.

We conduct numerical experiments on eight datasets: MNIST~\cite{LeCun1998MNIST}, Kuzushiji-MNIST~\cite{Clanuwat2018KMNIST}, Fashion-MNIST~\cite{Xiao2017FashionMNIST}, CIFAR-10~\cite{Krizhevsky2009CIFAR10}, Optdigits~\cite{Alpaydin1998Optdigits}, Pendigits~\cite{Alpaydin1996Pendigits}, Letter~\cite{Slate1991Letter}, and PMU-UD~\cite{Latif2018PMUUD}. 
We compare the proposed SD-Pcomp integration with methods that use only SD labels, methods that use only Pcomp labels, and methods that use a convex combination of the SD and Pcomp objective functions. 
We evaluate both classification accuracy and AUC. 
The comparisons with SD learning, Pcomp learning, and the convex combinations of their objective functions demonstrate the effectiveness of the proposed method, which integrates the two types of relational information.

\section{Related Work}

\subsection{Learning from Pointwise Weak Labels}

Weakly supervised learning trains classifiers from incomplete or indirect supervision when complete class labels are not available for individual training instances~\cite{sugiyama2022}. 
Several learning settings have been studied under different forms of weak supervision. Positive-Unlabeled (PU) learning uses positive and unlabeled instances~\cite{duplessis2015convex}, Unlabeled-Unlabeled (UU) learning uses unlabeled instances drawn from distributions with different class priors~\cite{lu2019minimal}, Partial-Label learning associates each instance with a set of candidate labels containing its true label~\cite{cour2011learning}, and Complementary-Label learning uses labels that specify classes to which an instance does not belong~\cite{ishida2017learning}. 
These methods construct classification objectives from incomplete label information observed for individual instances or from information about the distributions from which the instances are sampled.

\subsection{Learning from Pairwise Weak Labels}

In contrast, pairwise weakly supervised learning has also been studied, in which supervision is provided through relationships between two instances rather than through individual instances. 
Representative methods that use discrete relational information include Similarity/Dissimilarity (SD) learning, which uses whether two instances belong to the same or different classes~\cite{shimada2021}, and Pairwise Comparison (Pcomp) learning, which uses the relative ordering of two instances in terms of their likelihood of belonging to the positive class~\cite{feng2021}.

Pairwise weakly supervised learning with continuous-valued relational information has also been studied. 
Similarity-confidence learning uses the probability that two instances belong to the same class~\cite{cao2021}, whereas confidence-difference learning uses the difference between the probabilities that two instances belong to the positive class~\cite{wang2023}. 
In these settings, relationships between instances are represented by numerical values, whereas SD learning and Pcomp learning use only discrete relational information.

In this study, we consider a setting in which both an SD label and a Pcomp label are assigned to the same instance pair. 
This setting allows us to use two types of discrete relational information simultaneously: whether the two instances belong to the same or different classes, and their relative ordering in terms of their likelihood of belonging to the positive class. 
A simple approach is to construct the SD and Pcomp objective functions separately and take their convex combination. 
In contrast, we construct a single objective function that simultaneously preserves the structures of SD learning and Pcomp learning and directly integrates the two types of relational information from the same instance pair.

\section{Problem Setting and Existing Estimators}

\subsection{Binary Classification Setting}

Let the feature space be $\mathcal{X}\subset\mathbb{R}^d$ and the label space be $\mathcal{Y}=\{-1,+1\}$. The random variables $(X,Y)$ follow an unknown joint distribution $P$. 
We denote the class-prior probabilities by $\pi_+=\Pr(Y=+1)$ and $\pi_-=\Pr(Y=-1)$. 
Throughout this paper, we assume that $\pi_+,\pi_-\in(0,1)$, $\pi_++\pi_-=1$, and $\pi_+\neq\pi_-$, and we define $\Delta=\pi_+-\pi_-$. 
We also let $P_+$ and $P_-$ denote the distributions of $X\mid Y=+1$ and $X\mid Y=-1$, respectively.

Let $g:\mathcal{X}\to\mathbb{R}$ be a decision function and let $\ell:\mathbb{R}\times\mathcal{Y}\to\mathbb{R}_+$ be a nonnegative loss function. 
We assume that $\ell(z,+1)$ is non-increasing in $z$ and that $\ell(z,-1)$ is non-decreasing in $z$. Then, the classification risk is given by
\begin{align*}
R(g)
&=
\mathbb{E}\bigl[\ell(g(X),Y)\bigr]
\nonumber\\
&=
\pi_+\mathbb{E}_{P_+}\bigl[\ell(g(X),+1)\bigr]
+
\pi_-\mathbb{E}_{P_-}\bigl[\ell(g(X),-1)\bigr].
\end{align*}

\subsection{Pairwise Observations and Loss Representation}

Consider two independent samples $(X,Y)$ and $(X',Y')$. Let the conditional class probability be $\eta(x)=\Pr(Y=+1\mid X=x)$. 
A Pcomp label indicates which of the two instances is more likely to belong to the positive class. Specifically, we write the two instances as $X^+$ and $X^-$ and order the pair so that $\eta(X^+)\geq\eta(X^-)$. 
If $\eta(X)=\eta(X')$, we arbitrarily choose either instance as $X^+$. We also write the class labels corresponding to $X^+$ and $X^-$ as $Y^+$ and $Y^-$, respectively. 
Here, the superscripts $+$ and $-$ indicate the instances that are relatively more and less likely to belong to the positive class within the pair, respectively, and they do not represent the true class labels $+1$ and $-1$. 
Let $Q$ denote the distribution of the ordered pair $(X^+,X^-)$. In what follows, we refer to $X^+$ as the first element of the pair and $X^-$ as the second element.

For each pair, define the SD label as
\[
Z
=
\begin{cases}
\mathrm{S}, & Y^+=Y^-,\\
\mathrm{D}, & Y^+\neq Y^-
\end{cases}.
\]
The probabilities of the two labels are $\pi_{\mathrm{S}}=\Pr(Z=\mathrm{S})=\pi_+^2+\pi_-^2$ and $\pi_{\mathrm{D}}=\Pr(Z=\mathrm{D})=2\pi_+\pi_-$.

Let $Q$ denote the distribution of the ordered pair $(X^+,X^-)$. We draw $n$ pairs independently and identically from $Q$ and assume that the SD label of each pair is also observed. We write the observed dataset as
\[
\mathcal{D}
=
\bigl\{
(X_i^+,X_i^-,Z_i)
\bigr\}_{i=1}^{n}.
\]
For each $z\in\{\mathrm{S},\mathrm{D}\}$, we define $n_z=\sum_{i=1}^{n}\mathbf{1}\{Z_i=z\}$ and assume that $n_{\mathrm{S}},n_{\mathrm{D}}>0$.

For each $z\in\{\mathrm{S},\mathrm{D}\}$ and $t\in\{+,-\}$, define the population mean of the losses under the two possible class labels as
\begin{align*}
\boldsymbol{\phi}_{z,t}(g)
&=
\begin{pmatrix}
\mathbb{E}
\bigl[
\ell(g(X^t),+1)
\mid Z=z
\bigr]
\\[1mm]
\mathbb{E}
\bigl[
\ell(g(X^t),-1)
\mid Z=z
\bigr]
\end{pmatrix}
\end{align*}
and define the corresponding sample mean as
\begin{align*}
\widehat{\boldsymbol{\phi}}_{z,t}(g)
&=
\begin{pmatrix}
\displaystyle
\frac1{n_z}
\sum_{i:Z_i=z}
\ell(g(X_i^t),+1)
\\[3mm]
\displaystyle
\frac1{n_z}
\sum_{i:Z_i=z}
\ell(g(X_i^t),-1)
\end{pmatrix}.
\end{align*}

\subsection{SD Risk Estimator}

In existing SD learning, the observed samples are divided into similar pairs and dissimilar pairs, and different weights are used for the two types of pairs to estimate the classification risk. 
For the observed data considered in this paper, the SD estimator is defined as
\begin{align*}
&\widehat R_{\mathrm{SD}}(g)\nonumber\\
&=
\frac{\pi_{\mathrm{S}}}{2\Delta n_{\mathrm{S}}}
\sum_{i:Z_i=\mathrm{S}}
\left[
\begin{matrix}
\displaystyle
\pi_+
\bigl\{
\ell(g(X_i^+),+1)
+
\ell(g(X_i^-),+1)
\bigr\}
\\[1mm]
\displaystyle
-\pi_-
\bigl\{
\ell(g(X_i^+),-1)
+
\ell(g(X_i^-),-1)
\bigr\}
\end{matrix}
\right]
\nonumber\\
&\quad+
\frac{\pi_{\mathrm{D}}}{2\Delta n_{\mathrm{D}}}
\sum_{i:Z_i=\mathrm{D}}
\left[
\begin{matrix}
\displaystyle
-\pi_-
\bigl\{
\ell(g(X_i^+),+1)
+
\ell(g(X_i^-),+1)
\bigr\}
\\[1mm]
\displaystyle
+\pi_+
\bigl\{
\ell(g(X_i^+),-1)
+
\ell(g(X_i^-),-1)
\bigr\}
\end{matrix}
\right]
\end{align*}
\cite{shimada2021}.
This estimator distinguishes similar pairs from dissimilar pairs according to the SD labels, but it uses only the sum of the losses of the two elements in each pair. 
Therefore, the estimated value does not change even if the first and second elements are exchanged according to the Pcomp label, and the ordering information within each pair is not used. 
The estimator $\widehat R_{\mathrm{SD}}(g)$ is an unbiased estimator of the classification risk $R(g)$ based on SD labels~\cite{shimada2021}.

For later use, we write the weights for similar and dissimilar pairs as
\[
\boldsymbol{a}_{\mathrm{S}}
=
\frac{\pi_{\mathrm{S}}}{2\Delta}
\begin{pmatrix}
\pi_+\\
-\pi_-
\end{pmatrix},
\quad
\boldsymbol{a}_{\mathrm{D}}
=
\frac{\pi_{\mathrm{D}}}{2\Delta}
\begin{pmatrix}
-\pi_-\\
\pi_+
\end{pmatrix}.
\]
Then,
\begin{align}
&\widehat R_{\mathrm{SD}}(g)\nonumber\\
&=
\boldsymbol{a}_{\mathrm{S}}^{\top}
\bigl\{
\widehat{\boldsymbol{\phi}}_{\mathrm{S},+}(g)
+
\widehat{\boldsymbol{\phi}}_{\mathrm{S},-}(g)
\bigr\}
+
\boldsymbol{a}_{\mathrm{D}}^{\top}
\bigl\{
\widehat{\boldsymbol{\phi}}_{\mathrm{D},+}(g)
+
\widehat{\boldsymbol{\phi}}_{\mathrm{D},-}(g)
\bigr\}.
\label{eq:sd-estimator}
\end{align}
This representation shows that the SD estimator assigns different weights to similar and dissimilar pairs, while it assigns the same weight to the two elements of a pair with the same SD label.
Thus, the SD estimator reflects the pair type in its weights, but not the position within the pair.

\subsection{Pcomp Risk Estimator}

In existing Pcomp learning, the classification risk is estimated by distinguishing the first and second elements of each ordered pair.
For the observations considered here, the Pcomp risk estimator is defined as
\begin{align*}
\widehat R_{\mathrm{PC}}(g)
&=
\frac1n
\sum_{i=1}^{n}
\left[
\begin{matrix}
\displaystyle
\ell(g(X_i^+),+1)
-
\pi_+\ell(g(X_i^+),-1)
\\[1mm]
\displaystyle
-\pi_-\ell(g(X_i^-),+1)
+
\ell(g(X_i^-),-1)
\end{matrix}
\right]
\end{align*}
\cite{feng2021}.
This estimator uses different combinations of losses for the first and second elements of each pair.
Therefore, it uses the within-pair ordering represented by the Pcomp label, while its form does not depend on whether the pair is similar or dissimilar.
The estimator $\widehat R_{\mathrm{PC}}(g)$ is an unbiased estimator of the classification risk $R(g)$ based on Pcomp labels~\cite{feng2021}.

For later use, we write the weights for the first and second elements as
\[
\boldsymbol{b}_+
=
\begin{pmatrix}
1\\
-\pi_+
\end{pmatrix},
\quad
\boldsymbol{b}_-
=
\begin{pmatrix}
-\pi_-\\
1
\end{pmatrix}.
\]
Then,
\begin{align}
\widehat R_{\mathrm{PC}}(g)
&=
\frac{n_{\mathrm{S}}}{n}
\Bigl[
\boldsymbol{b}_+^{\top}
\widehat{\boldsymbol{\phi}}_{\mathrm{S},+}(g)
+
\boldsymbol{b}_-^{\top}
\widehat{\boldsymbol{\phi}}_{\mathrm{S},-}(g)
\Bigr]
\nonumber\\
&\quad+
\frac{n_{\mathrm{D}}}{n}
\Bigl[
\boldsymbol{b}_+^{\top}
\widehat{\boldsymbol{\phi}}_{\mathrm{D},+}(g)
+
\boldsymbol{b}_-^{\top}
\widehat{\boldsymbol{\phi}}_{\mathrm{D},-}(g)
\Bigr].
\label{eq:pcomp-estimator-vector}
\end{align}
This representation shows that the Pcomp estimator assigns different weights to the first and second elements, while it assigns the same weight to elements at the same position regardless of whether the pair is similar or dissimilar.
Thus, the Pcomp estimator distinguishes the position within the pair, but does not reflect the pair type in its weights.

\section{Joint Learning from SD and Pcomp Labels}
\label{sec:joint-learning}

The SD and Pcomp estimators introduced in the previous section use different types of relational information from the same observed data.
A simple way to integrate them is to use the convex combination
\begin{equation}
\widehat R_{\mathrm{Convex}}(g;\gamma)
=
\gamma\widehat R_{\mathrm{SD}}(g)
+
(1-\gamma)\widehat R_{\mathrm{PC}}(g),
\quad
\gamma\in[0,1].
\label{eq:convex-objective}
\end{equation}
This approach constructs the two existing estimators separately and then combines them.
In contrast, our objective directly couples the SD and Pcomp labels assigned to the same pair.

\subsection{Construction of the Objective Function}

The SD estimator assigns different weights to similar and dissimilar pairs, but does not distinguish the two positions within each pair.
In contrast, the Pcomp estimator assigns different weights to the two positions within each pair, but does not distinguish between similar and dissimilar pairs.
We therefore construct four weights that distinguish both the pair type and the position within the pair.

For each $t\in\{+,-\}$, let
$c_{\mathrm{S},t},c_{\mathrm{D},t}\in\mathbb{R}$ be scalar coefficients, and define
\begin{align}
\boldsymbol{a}_{\mathrm{S},t}
=
\frac{c_{\mathrm{S},t}}{\Delta}
\begin{pmatrix}
\pi_+\\
-\pi_-
\end{pmatrix},
\quad
\boldsymbol{a}_{\mathrm{D},t}
=
\frac{c_{\mathrm{D},t}}{\Delta}
\begin{pmatrix}
-\pi_-\\
\pi_+
\end{pmatrix}.
\label{eq:sdpc-weight}
\end{align}
This preserves the relative weights of the two losses in the SD estimator for both similar and dissimilar pairs.
We further impose
\begin{equation}
\boldsymbol{a}_{\mathrm{S},t}
+
\boldsymbol{a}_{\mathrm{D},t}
=
\boldsymbol{b}_t,
\quad
t\in\{+,-\},
\label{eq:sdpc-compatibility}
\end{equation}
so that the weights of the Pcomp estimator are recovered at each position.

\begin{theorem}
\label{thm:sdpc-weights}
Under the assumption $\pi_+\neq\pi_-$, the four weights satisfying
Eqs.~\eqref{eq:sdpc-weight} and \eqref{eq:sdpc-compatibility}
are uniquely determined as
\begin{align}
\boldsymbol{a}_{\mathrm{S},+}
&=
\frac{\pi_+^2}{\Delta}
\begin{pmatrix}
\pi_+\\
-\pi_-
\end{pmatrix},
\label{eq:a-s-plus}\\
\boldsymbol{a}_{\mathrm{S},-}
&=
\frac{\pi_-^2}{\Delta}
\begin{pmatrix}
\pi_+\\
-\pi_-
\end{pmatrix},
\label{eq:a-s-minus}\\
\boldsymbol{a}_{\mathrm{D},+}
&=
\frac{\pi_- - \pi_+^2}{\Delta}
\begin{pmatrix}
-\pi_-\\
\pi_+
\end{pmatrix},
\label{eq:a-d-plus}\\
\boldsymbol{a}_{\mathrm{D},-}
&=
\frac{\pi_+ - \pi_-^2}{\Delta}
\begin{pmatrix}
-\pi_-\\
\pi_+
\end{pmatrix}.
\label{eq:a-d-minus}
\end{align}
Moreover,
\begin{align}
\frac{
\boldsymbol{a}_{\mathrm{S},+}
+
\boldsymbol{a}_{\mathrm{S},-}
}{2}
&=
\boldsymbol{a}_{\mathrm{S}},
&
\frac{
\boldsymbol{a}_{\mathrm{D},+}
+
\boldsymbol{a}_{\mathrm{D},-}
}{2}
&=
\boldsymbol{a}_{\mathrm{D}}.
\label{eq:sd-recovery}
\end{align}
\end{theorem}

The proof is provided in Appendix~\ref{app:sdpc-construction}.
From Eq.~\eqref{eq:sdpc-compatibility}, combining the weights for similar and dissimilar pairs at each position yields the corresponding weight of the Pcomp estimator.
From Eq.~\eqref{eq:sd-recovery}, averaging the weights at the two positions for each pair type yields the corresponding weight of the SD estimator.
Thus, the four weights reflect the weighting structures of both the SD and Pcomp estimators.

Using the four weights obtained in Theorem~\ref{thm:sdpc-weights}, we define the SD-Pcomp objective function as
\begin{align}
\widehat J_{\mathrm{SDPC}}(g)
&=
\boldsymbol{a}_{\mathrm{S},+}^{\top}
\widehat{\boldsymbol{\phi}}_{\mathrm{S},+}(g)
+
\boldsymbol{a}_{\mathrm{S},-}^{\top}
\widehat{\boldsymbol{\phi}}_{\mathrm{S},-}(g)
\nonumber\\
&\quad+
\boldsymbol{a}_{\mathrm{D},+}^{\top}
\widehat{\boldsymbol{\phi}}_{\mathrm{D},+}(g)
+
\boldsymbol{a}_{\mathrm{D},-}^{\top}
\widehat{\boldsymbol{\phi}}_{\mathrm{D},-}(g).
\label{eq:sdpc-objective}
\end{align}
This objective assigns different weights not only to similar and dissimilar pairs but also to the first and second elements within each pair.
Thus, the weights of the individual loss terms reflect both the pair type represented by the SD label and the within-pair ordering represented by the Pcomp label.
The proposed objective can be decomposed with respect to either the SD estimator or the Pcomp estimator.
In the following, we use these two decompositions to show what information the proposed objective adds to each estimator.

\subsection{Interpretation Based on the SD Estimator}
\label{subsec:sd-label-meaning}

We first consider the proposed objective from the perspective of the SD estimator.
The SD estimator uses the pair-type information provided by the SD labels through different coefficients for similar and dissimilar pairs.
However, it treats the first and second elements of each pair symmetrically and therefore does not use the within-pair ordering information provided by the Pcomp labels.
This relationship can be characterized by the difference between the proposed objective and the SD estimator, as in the following theorem.
The proof is provided in Appendix \ref{app:proof-sdpc-sd-decomposition}.

\begin{theorem}
\label{thm:sdpc-sd-decomposition}
The SD-Pcomp objective function defined in Eq.~\eqref{eq:sdpc-objective} can be expressed as
\begin{equation*}
\widehat J_{\mathrm{SDPC}}(g)
=
\widehat R_{\mathrm{SD}}(g)
+
\widehat B_{\mathrm{SDPC}}(g),
\end{equation*}
where
\begin{align}
\widehat B_{\mathrm{SDPC}}(g)
&=
\frac12
\begin{pmatrix}
\pi_+\\
-\pi_-
\end{pmatrix}^{\top}
\left\{
\widehat{\boldsymbol{\phi}}_{\mathrm{S},+}(g)
-
\widehat{\boldsymbol{\phi}}_{\mathrm{S},-}(g)
\right\}
\nonumber\\
&\quad-
\begin{pmatrix}
-\pi_-\\
\pi_+
\end{pmatrix}^{\top}
\left\{
\widehat{\boldsymbol{\phi}}_{\mathrm{D},+}(g)
-
\widehat{\boldsymbol{\phi}}_{\mathrm{D},-}(g)
\right\}.
\label{eq:empirical-order-term}
\end{align}
\end{theorem}

Equation~\eqref{eq:empirical-order-term} involves differences between the losses at the two positions for both similar and dissimilar pairs.
Since the SD estimator treats the two positions symmetrically, $\widehat B_{\mathrm{SDPC}}(g)$ represents the additional within-pair ordering information provided by the Pcomp labels.

To examine the effect of this additional term, define
\begin{align*}
b_{\mathrm{S}}(g;x^+,x^-)
&=
\frac12
\left[
\pi_+
\left\{
\ell(g(x^+),+1)
-
\ell(g(x^-),+1)
\right\}
\right.
\nonumber\\
&\quad\left.
-
\pi_-
\left\{
\ell(g(x^+),-1)
-
\ell(g(x^-),-1)
\right\}
\right]
\end{align*}
for a similar pair, and
\begin{align*}
b_{\mathrm{D}}(g;x^+,x^-)
&=
\pi_-
\left\{
\ell(g(x^+),+1)
-
\ell(g(x^-),+1)
\right\}
\nonumber\\
&\quad-
\pi_+
\left\{
\ell(g(x^+),-1)
-
\ell(g(x^-),-1)
\right\}
\end{align*}
for a dissimilar pair.
Then,
\begin{align}
\widehat B_{\mathrm{SDPC}}(g)
&=
\frac1{n_{\mathrm{S}}}
\sum_{i:Z_i=\mathrm{S}}
b_{\mathrm{S}}
\left(
g;X_i^+,X_i^-
\right)
\nonumber\\
&\quad+
\frac1{n_{\mathrm{D}}}
\sum_{i:Z_i=\mathrm{D}}
b_{\mathrm{D}}
\left(
g;X_i^+,X_i^-
\right).
\label{eq:empirical-order-pair-form}
\end{align}

Under the monotonicity assumption on the loss function introduced above, if $g(x^+)\geq g(x^-)$, then
\[
b_{\mathrm{S}}(g;x^+,x^-)\leq0,
\quad
b_{\mathrm{D}}(g;x^+,x^-)\leq0,
\]
whereas if $g(x^+)\leq g(x^-)$, then
\[
b_{\mathrm{S}}(g;x^+,x^-)\geq0,
\quad
b_{\mathrm{D}}(g;x^+,x^-)\geq0.
\]
Thus, $\widehat B_{\mathrm{SDPC}}(g)$ decreases the objective for score orderings consistent with the Pcomp labels and increases it for the opposite orderings.
Therefore, $\widehat B_{\mathrm{SDPC}}(g)$ can be interpreted as a data-dependent regularization term that incorporates Pcomp ordering information into SD learning.

\subsection{Interpretation Based on the Pcomp Estimator}
\label{subsec:pcomp-label-meaning}

We next consider the proposed objective from the perspective of the Pcomp estimator.
The Pcomp estimator uses the within-pair ordering information provided by the Pcomp labels through different coefficients for the first and second elements.
However, it uses the same coefficients for similar and dissimilar pairs at each position and therefore does not use the pair-type information provided by the SD labels.
This relationship can be characterized by the difference between the proposed objective and the Pcomp estimator, as in the following theorem.
The proof is provided in Appendix~\ref{app:proof-sdpc-pcomp-decomposition}.

\begin{theorem}
\label{thm:sdpc-pcomp-decomposition}
The SD-Pcomp objective function defined in Eq.~\eqref{eq:sdpc-objective} can be decomposed as
\begin{equation*}
\widehat J_{\mathrm{SDPC}}(g)
=
\widehat R_{\mathrm{PC}}(g)
+
\widehat C_{\mathrm{SDPC}}(g),
\end{equation*}
where
\begin{align}
\widehat C_{\mathrm{SDPC}}(g)
&=
\boldsymbol c_+^{\top}
\left\{
\widehat{\boldsymbol{\phi}}_{\mathrm{S},+}(g)
-
\widehat{\boldsymbol{\phi}}_{\mathrm{D},+}(g)
\right\}
\nonumber\\
&\quad+
\boldsymbol c_-^{\top}
\left\{
\widehat{\boldsymbol{\phi}}_{\mathrm{S},-}(g)
-
\widehat{\boldsymbol{\phi}}_{\mathrm{D},-}(g)
\right\},
\label{eq:empirical-sd-label-term}
\end{align}
with
\begin{equation*}
w_{\mathrm{S}}
=
\frac{n_{\mathrm{S}}}{n},
\qquad
w_{\mathrm{D}}
=
\frac{n_{\mathrm{D}}}{n},
\end{equation*}
and, for each $t\in\{+,-\}$,
\begin{align}
\boldsymbol c_t
=
w_{\mathrm{D}}\boldsymbol a_{\mathrm{S},t}
-
w_{\mathrm{S}}\boldsymbol a_{\mathrm{D},t}.
\label{eq:ct-definition}
\end{align}
\end{theorem}

From Eq.~\eqref{eq:empirical-sd-label-term}, $\widehat C_{\mathrm{SDPC}}(g)$ reflects the differences between similar and dissimilar pairs at each of the two positions.
These differences are not used by the Pcomp estimator, which assigns the same weights regardless of the pair type.
Thus, $\widehat C_{\mathrm{SDPC}}(g)$ can be interpreted as a correction term that adds the pair-type information provided by the SD labels to the Pcomp estimator.

This interpretation is complementary to the SD-based interpretation in the previous subsection.
From the perspective of the SD estimator, the proposed objective adds ordering information to pair-type information.
From the perspective of the Pcomp estimator, it adds pair-type information to ordering information.
Thus, the proposed objective combines the information that the SD and Pcomp estimators do not use individually.

\section{Experiments}
\label{SDPC:sec:experiment}

In this section, we evaluate the empirical performance of the SD-Pcomp objective.
We compare SD-Pcomp learning with SD learning, Pcomp learning, and a convex combination of the SD and Pcomp estimators.
The comparison with SD learning examines the effect of additionally using the within-pair ordering information provided by Pcomp labels, while the comparison with Pcomp learning examines the effect of additionally using the distinction between similar and dissimilar pairs.
The comparison with the convex combination examines the difference between combining two separately constructed objective functions and directly incorporating the two types of relational information from the same instance pair into a single objective function.

\subsection{Experimental Setup}
\label{SDPC:sec:exp_setup}

In the experiments, we compare how the methods use the same pairwise supervision.
We generate SD and Pcomp labels from the true class labels in each dataset and use only these pairwise labels for training.
Thus, SD, Pcomp, Convex, and SD-Pcomp use the same instance pairs and differ only in how they incorporate the relational information associated with these pairs into their objective functions.

To facilitate comparison with prior work, we follow the experimental setting of Feng et al.~\cite{feng2021} and use MNIST~\cite{LeCun1998MNIST}, Kuzushiji-MNIST~\cite{Clanuwat2018KMNIST}, Fashion-MNIST~\cite{Xiao2017FashionMNIST}, and CIFAR-10~\cite{Krizhevsky2009CIFAR10}.
We also evaluate Optdigits~\cite{Alpaydin1998Optdigits}, Pendigits~\cite{Alpaydin1996Pendigits}, Letter~\cite{Slate1991Letter}, and PMU-UD~\cite{Latif2018PMUUD}.
For CIFAR-10, we use ResNet34~\cite{He2016ResNet}.
For the other datasets, we use an MLP with three hidden layers of width 300, ReLU activation functions~\cite{Nair2010ReLU}, and batch normalization~\cite{Ioffe2015BatchNorm}.
We use the logistic loss for $\ell$.

To examine the effect of the class prior, we set
$\pi_+\in\{0.1,0.4,0.7\}$ for each dataset and assume that $\pi_+$ is known.
For each setting, we repeat the sampling and training five times and report the mean and standard deviation of classification accuracy and AUC.

The main objectives considered for comparison are the proposed SD-Pcomp objective, SD learning based only on SD labels, Pcomp learning based only on Pcomp labels, and Convex, which combines the SD and Pcomp estimators through Eq.~\eqref{eq:convex-objective}.
For Convex, we use $\gamma\in\{0.2,0.5,0.8\}$.

Empirical objectives with negative coefficients can take negative values for finite samples, which may lead to undesirable overfitting when they are directly minimized. Motivated by the correction techniques used in Pcomp learning~\cite{feng2021}, we therefore also evaluate practical corrections based on the ReLU and absolute-value functions.
In the tables, ReLU and ABS denote the corresponding corrections, while the method name without a suffix denotes the original objective without correction.
These corrections are not part of the construction of the SD-Pcomp objective itself, but are included as auxiliary comparisons to examine its finite-sample behavior.
We also report standard supervised learning with the true class labels as a reference.

\begin{table*}[t]
\caption{Classification accuracy on the benchmark test set with $\pi_+=0.7$ averaged over five random seeds, with mean and standard deviation (mean$\pm$std). The highest score excluding supervised learning is shown in bold.}
\label{SDPC:tab:exp_accuracy_benchmark}
\begin{center}
\begin{tabular}{clcccc}
\toprule
Class Prior & \multicolumn{1}{c}{Method} & MNIST & Kuzushiji & Fashion & CIFAR10 \\
\midrule
& SD-Pcomp & 0.902 $\pm$ 0.003 & 0.756 $\pm$ 0.006 & 0.849 $\pm$ 0.000 & \textbf{0.820 $\pm$ 0.069} \\
& SD-Pcomp-ReLU & 0.892 $\pm$ 0.024 & 0.624 $\pm$ 0.101 & 0.807 $\pm$ 0.040 & 0.656 $\pm$ 0.013 \\
& SD-Pcomp-ABS & 0.760 $\pm$ 0.019 & 0.578 $\pm$ 0.035 & 0.576 $\pm$ 0.027 & 0.405 $\pm$ 0.028 \\
\cmidrule(lr){2-6}
& Convex ($\gamma=0.2$) & 0.718 $\pm$ 0.008 & 0.707 $\pm$ 0.016 & 0.672 $\pm$ 0.006 & 0.683 $\pm$ 0.032 \\
& Convex ($\gamma=0.5$) & 0.744 $\pm$ 0.025 & 0.713 $\pm$ 0.006 & 0.676 $\pm$ 0.009 & 0.667 $\pm$ 0.052 \\
& Convex ($\gamma=0.8$) & 0.794 $\pm$ 0.040 & 0.687 $\pm$ 0.008 & 0.811 $\pm$ 0.031 & 0.690 $\pm$ 0.020 \\
\cmidrule(lr){2-6}
& Convex ($\gamma=0.2$)-ReLU & \textbf{0.909 $\pm$ 0.006} & \textbf{0.785 $\pm$ 0.025} & 0.895 $\pm$ 0.012 & 0.693 $\pm$ 0.046 \\
& Convex ($\gamma=0.5$)-ReLU & 0.908 $\pm$ 0.009 & 0.784 $\pm$ 0.033 & 0.893 $\pm$ 0.011 & 0.698 $\pm$ 0.052 \\
$\pi_+=0.7$ & Convex ($\gamma=0.8$)-ReLU & 0.909 $\pm$ 0.009 & 0.781 $\pm$ 0.021 & 0.860 $\pm$ 0.042 & 0.673 $\pm$ 0.016 \\
\cmidrule(lr){2-6}
& Convex ($\gamma=0.2$)-ABS & 0.868 $\pm$ 0.034 & 0.581 $\pm$ 0.089 & 0.924 $\pm$ 0.009 & 0.758 $\pm$ 0.027 \\
& Convex ($\gamma=0.5$)-ABS & 0.905 $\pm$ 0.003 & 0.735 $\pm$ 0.040 & \textbf{0.925 $\pm$ 0.012} & 0.751 $\pm$ 0.042 \\
& Convex ($\gamma=0.8$)-ABS & 0.891 $\pm$ 0.005 & 0.746 $\pm$ 0.027 & 0.885 $\pm$ 0.020 & 0.692 $\pm$ 0.009 \\
\cmidrule(lr){2-6}
& SD & 0.839 $\pm$ 0.027 & 0.722 $\pm$ 0.031 & 0.881 $\pm$ 0.052 & 0.680 $\pm$ 0.014 \\
& Pcomp & 0.706 $\pm$ 0.010 & 0.696 $\pm$ 0.012 & 0.672 $\pm$ 0.008 & 0.698 $\pm$ 0.004 \\
& Pcomp-ReLU & 0.662 $\pm$ 0.034 & 0.675 $\pm$ 0.079 & 0.554 $\pm$ 0.077 & 0.389 $\pm$ 0.170 \\
& Pcomp-ABS & 0.607 $\pm$ 0.022 & 0.768 $\pm$ 0.034 & 0.344 $\pm$ 0.083 & 0.306 $\pm$ 0.019 \\
\cmidrule(lr){2-6}
& Supervised & 0.990 $\pm$ 0.001 & 0.940 $\pm$ 0.001 & 0.992 $\pm$ 0.001 & 0.887 $\pm$ 0.005 \\
\bottomrule
\end{tabular}%
\end{center}
\vskip -0.1in
\end{table*}

\begin{table*}[!t]
\caption{Classification accuracy on the UCI test set with $\pi_+=0.7$ averaged over five random seeds, with mean and standard deviation (mean$\pm$std). The highest score excluding supervised learning is shown in bold.}
\label{SDPC:tab:exp_accuracy_uci}
\begin{center}
\begin{tabular}{clcccc}
\toprule
Class Prior & \multicolumn{1}{c}{Method} & Optdigits & Pendigits & Letter & PMU-UD \\
\midrule
& SD-Pcomp & \textbf{0.904 $\pm$ 0.021} & 0.863 $\pm$ 0.017 & 0.734 $\pm$ 0.013 & 0.845 $\pm$ 0.005 \\
& SD-Pcomp-ReLU & 0.831 $\pm$ 0.079 & \textbf{0.936 $\pm$ 0.016} & 0.775 $\pm$ 0.026 & \textbf{0.971 $\pm$ 0.010} \\
& SD-Pcomp-ABS & 0.704 $\pm$ 0.025 & 0.690 $\pm$ 0.032 & 0.668 $\pm$ 0.046 & 0.709 $\pm$ 0.010 \\
\cmidrule(lr){2-6}
& Convex ($\gamma=0.2$) & 0.724 $\pm$ 0.012 & 0.844 $\pm$ 0.028 & 0.697 $\pm$ 0.013 & 0.760 $\pm$ 0.018 \\
& Convex ($\gamma=0.5$) & 0.795 $\pm$ 0.042 & 0.875 $\pm$ 0.032 & 0.683 $\pm$ 0.033 & 0.860 $\pm$ 0.030 \\
& Convex ($\gamma=0.8$) & 0.787 $\pm$ 0.051 & 0.800 $\pm$ 0.031 & 0.671 $\pm$ 0.028 & 0.930 $\pm$ 0.034 \\
\cmidrule(lr){2-6}
& Convex ($\gamma=0.2$)-ReLU & 0.875 $\pm$ 0.024 & 0.895 $\pm$ 0.019 & 0.744 $\pm$ 0.028 & 0.882 $\pm$ 0.015 \\
& Convex ($\gamma=0.5$)-ReLU & 0.844 $\pm$ 0.038 & 0.860 $\pm$ 0.032 & 0.761 $\pm$ 0.030 & 0.833 $\pm$ 0.033 \\
$\pi_+=0.7$ & Convex ($\gamma=0.8$)-ReLU & 0.869 $\pm$ 0.025 & 0.898 $\pm$ 0.014 & 0.790 $\pm$ 0.029 & 0.919 $\pm$ 0.028 \\
\cmidrule(lr){2-6}
& Convex ($\gamma=0.2$)-ABS & 0.889 $\pm$ 0.017 & 0.858 $\pm$ 0.026 & \textbf{0.805 $\pm$ 0.029} & 0.921 $\pm$ 0.034 \\
& Convex ($\gamma=0.5$)-ABS & 0.872 $\pm$ 0.018 & 0.892 $\pm$ 0.021 & 0.721 $\pm$ 0.030 & 0.887 $\pm$ 0.020 \\
& Convex ($\gamma=0.8$)-ABS & 0.853 $\pm$ 0.015 & 0.877 $\pm$ 0.009 & 0.774 $\pm$ 0.022 & 0.855 $\pm$ 0.041 \\
\cmidrule(lr){2-6}
& SD & 0.871 $\pm$ 0.045 & 0.828 $\pm$ 0.028 & 0.694 $\pm$ 0.029 & 0.944 $\pm$ 0.019 \\
& Pcomp & 0.711 $\pm$ 0.009 & 0.802 $\pm$ 0.013 & 0.690 $\pm$ 0.007 & 0.733 $\pm$ 0.014 \\
& Pcomp-ReLU & 0.613 $\pm$ 0.071 & 0.672 $\pm$ 0.076 & 0.633 $\pm$ 0.042 & 0.623 $\pm$ 0.075 \\
& Pcomp-ABS & 0.693 $\pm$ 0.031 & 0.690 $\pm$ 0.058 & 0.674 $\pm$ 0.016 & 0.697 $\pm$ 0.054 \\
\cmidrule(lr){2-6}
& Supervised & 0.987 $\pm$ 0.002 & 0.996 $\pm$ 0.002 & 0.975 $\pm$ 0.002 & 0.991 $\pm$ 0.004 \\
\bottomrule
\end{tabular}%
\end{center}
\vskip -0.1in
\end{table*}

\begin{table*}[!t]
\caption{AUC values on the benchmark test set with $\pi_+=0.7$ averaged over five random seeds, with mean and standard deviation (mean$\pm$std). The highest score is shown in bold.}
\label{SDPC:tab:exp_AUC_benchmark}
\begin{center}
\begin{tabular}{clcccc}
\toprule
Class Prior & \multicolumn{1}{c}{Method} & MNIST & Kuzushiji & Fashion & CIFAR10 \\
\midrule
& SD-Pcomp & 0.958 $\pm$ 0.005 & 0.825 $\pm$ 0.008 & 0.903 $\pm$ 0.007 & \textbf{0.930 $\pm$ 0.002} \\
& SD-Pcomp-ReLU & \textbf{0.985 $\pm$ 0.004} & 0.826 $\pm$ 0.069 & 0.957 $\pm$ 0.020 & 0.865 $\pm$ 0.016 \\
& SD-Pcomp-ABS & 0.792 $\pm$ 0.048 & 0.717 $\pm$ 0.025 & 0.606 $\pm$ 0.094 & 0.564 $\pm$ 0.050 \\
\cmidrule(lr){2-6}
& Convex ($\gamma=0.2$)-ReLU & 0.976 $\pm$ 0.006 & 0.856 $\pm$ 0.016 & 0.935 $\pm$ 0.019 & 0.867 $\pm$ 0.029 \\
& Convex ($\gamma=0.5$)-ReLU & 0.981 $\pm$ 0.003 & 0.868 $\pm$ 0.026 & 0.951 $\pm$ 0.010 & 0.864 $\pm$ 0.035 \\
$\pi_+=0.7$ & Convex ($\gamma=0.8$)-ReLU & 0.980 $\pm$ 0.004 & \textbf{0.874 $\pm$ 0.018} & 0.939 $\pm$ 0.052 & 0.746 $\pm$ 0.007 \\
\cmidrule(lr){2-6}
& Convex ($\gamma=0.2$)-ABS & 0.973 $\pm$ 0.003 & 0.746 $\pm$ 0.077 & \textbf{0.982 $\pm$ 0.005} & 0.834 $\pm$ 0.087 \\
& Convex ($\gamma=0.5$)-ABS & 0.978 $\pm$ 0.003 & 0.856 $\pm$ 0.041 & 0.980 $\pm$ 0.006 & 0.746 $\pm$ 0.118 \\
& Convex ($\gamma=0.8$)-ABS & 0.980 $\pm$ 0.003 & 0.868 $\pm$ 0.034 & 0.963 $\pm$ 0.004 & 0.757 $\pm$ 0.013 \\
\cmidrule(lr){2-6}
& SD & 0.952 $\pm$ 0.004 & 0.803 $\pm$ 0.026 & 0.966 $\pm$ 0.013 & 0.728 $\pm$ 0.030 \\
& Pcomp & 0.747 $\pm$ 0.026 & 0.705 $\pm$ 0.029 & 0.711 $\pm$ 0.039 & 0.497 $\pm$ 0.115 \\
& Pcomp-ReLU & 0.839 $\pm$ 0.083 & 0.723 $\pm$ 0.033 & 0.643 $\pm$ 0.162 & 0.457 $\pm$ 0.144 \\
& Pcomp-ABS & 0.643 $\pm$ 0.016 & 0.823 $\pm$ 0.046 & 0.407 $\pm$ 0.046 & 0.406 $\pm$ 0.098 \\
\bottomrule
\end{tabular}%
\end{center}
\vskip -0.1in
\end{table*}

\begin{table*}[!t]
\caption{AUC value on the UCI test set with $\pi_+=0.7$ averaged over five random seeds, with mean and standard deviation (mean$\pm$std). The highest score is shown in bold.}
\label{SDPC:tab:exp_AUC_uci}
\begin{center}
\begin{tabular}{clcccc}
\toprule
Class Prior & \multicolumn{1}{c}{Method} & Optdigits & Pendigits & Letter & PMU-UD \\
\midrule
& SD-Pcomp & 0.957 $\pm$ 0.010 & 0.965 $\pm$ 0.005 & 0.771 $\pm$ 0.016 & \textbf{0.992 $\pm$ 0.002} \\
& SD-Pcomp-ReLU & \textbf{0.976 $\pm$ 0.006} & \textbf{0.979 $\pm$ 0.003} & 0.898 $\pm$ 0.008 & 0.992 $\pm$ 0.004 \\
& SD-Pcomp-ABS & 0.714 $\pm$ 0.042 & 0.684 $\pm$ 0.038 & 0.779 $\pm$ 0.024 & 0.696 $\pm$ 0.045 \\
\cmidrule(lr){2-6}
& Convex ($\gamma=0.2$)-ReLU & 0.958 $\pm$ 0.006 & 0.971 $\pm$ 0.005 & 0.868 $\pm$ 0.024 & 0.938 $\pm$ 0.018 \\
& Convex ($\gamma=0.5$)-ReLU & 0.960 $\pm$ 0.006 & 0.966 $\pm$ 0.004 & 0.883 $\pm$ 0.015 & 0.953 $\pm$ 0.008 \\
$\pi_+=0.7$ & Convex ($\gamma=0.8$)-ReLU & 0.966 $\pm$ 0.005 & 0.972 $\pm$ 0.006 & 0.893 $\pm$ 0.016 & 0.961 $\pm$ 0.013 \\
\cmidrule(lr){2-6}
& Convex ($\gamma=0.2$)-ABS & 0.970 $\pm$ 0.008 & 0.966 $\pm$ 0.008 & \textbf{0.914 $\pm$ 0.013} & 0.980 $\pm$ 0.005 \\
& Convex ($\gamma=0.5$)-ABS & 0.964 $\pm$ 0.004 & 0.973 $\pm$ 0.008 & 0.855 $\pm$ 0.028 & 0.967 $\pm$ 0.003 \\
& Convex ($\gamma=0.8$)-ABS & 0.973 $\pm$ 0.007 & 0.977 $\pm$ 0.005 & 0.901 $\pm$ 0.016 & 0.971 $\pm$ 0.015 \\
\cmidrule(lr){2-6}
& SD & 0.950 $\pm$ 0.010 & 0.903 $\pm$ 0.068 & 0.769 $\pm$ 0.037 & 0.975 $\pm$ 0.008 \\
& Pcomp & 0.714 $\pm$ 0.034 & 0.827 $\pm$ 0.027 & 0.665 $\pm$ 0.013 & 0.885 $\pm$ 0.013 \\
& Pcomp-ReLU & 0.795 $\pm$ 0.067 & 0.878 $\pm$ 0.012 & 0.750 $\pm$ 0.031 & 0.842 $\pm$ 0.033 \\
& Pcomp-ABS & 0.864 $\pm$ 0.032 & 0.819 $\pm$ 0.029 & 0.773 $\pm$ 0.012 & 0.867 $\pm$ 0.046 \\
\bottomrule
\end{tabular}%
\end{center}
\vskip -0.1in
\end{table*}

\subsection{Experimental Results}
\label{SDPC:sec:exp_acc}

Tables~\ref{SDPC:tab:exp_accuracy_benchmark} and
\ref{SDPC:tab:exp_accuracy_uci} report the classification accuracy for $\pi_+=0.7$.
Tables~\ref{SDPC:tab:exp_AUC_benchmark} and
\ref{SDPC:tab:exp_AUC_uci} report the AUC under the same setting.

We first compare SD-Pcomp with SD learning.
Without correction, SD-Pcomp achieves higher classification accuracy than SD on most of the eight datasets.
The difference is more pronounced for AUC, for which SD-Pcomp outperforms SD on most datasets.
As shown in the previous section, the SD-Pcomp objective can be decomposed into the SD estimator and an additional term that represents the within-pair ordering information provided by Pcomp labels.
These results are therefore consistent with the interpretation that incorporating ordering information, which SD learning does not use, can improve classification performance.
In particular, the improvement in AUC suggests that the ordering information provided by Pcomp labels also affects the ranking of classifier scores.

We next compare SD-Pcomp with Pcomp learning.
Without correction, SD-Pcomp achieves higher classification accuracy than Pcomp on all eight datasets.
The same tendency is observed for AUC, for which SD-Pcomp also outperforms Pcomp on all datasets.
As shown in the previous section, the SD-Pcomp objective can also be decomposed into the Pcomp estimator and an additional term that reflects the distinction between similar and dissimilar pairs.
These results are therefore consistent with the interpretation that using pair-type information, which Pcomp learning does not distinguish, can be beneficial compared with using only the ordering information provided by Pcomp labels.

We also compare SD-Pcomp with Convex.
Both methods use SD and Pcomp labels, but they incorporate the two types of information differently.
Convex combines the separately constructed SD and Pcomp estimators, whereas SD-Pcomp directly constructs weights that reflect both the similar/dissimilar distinction and the position within each pair.
Although SD-Pcomp does not outperform Convex in every setting, it achieves competitive performance across the datasets and evaluation metrics, and in several cases it performs as well as or better than Convex with all three values of $\gamma$.
These results demonstrate the effectiveness of SD-Pcomp and show that it provides a viable alternative to convex combinations for integrating SD and Pcomp information.

We also evaluated the settings with $\pi_+=0.1$ and $\pi_+=0.4$.
These settings showed performance differences among SD-Pcomp, SD, Pcomp, and Convex similar to those observed for $\pi_+=0.7$, depending on the relational information used and how it is integrated.
The detailed results are omitted due to space limitations.

\section{Conclusion}

In this study, we considered SD-Pcomp classification, which uses both SD and Pcomp labels assigned to the same instance pairs, and proposed a method that directly incorporates both types of information into a single objective function.
The proposed objective uses weights that reflect both the distinction between similar and dissimilar pairs and the ordering within each pair.
We also showed that the objective admits two decompositions: an SD estimator plus a term that represents the ordering information provided by Pcomp labels, and a Pcomp estimator plus a term that represents the pair-type information provided by SD labels.
Experiments on eight datasets showed that SD-Pcomp generally achieved better performance than SD learning or Pcomp learning alone.
Although SD-Pcomp did not outperform the convex combinations in every setting, it achieved competitive performance and provides a distinct alternative for integrating the same two types of relational information into a single objective function.

\bibliographystyle{IEEEtran}
\bibliography{references}

@article{shimada2021,
  author  = {Shimada, Takuya and Bao, Han and Sato, Issei and Sugiyama, Masashi},
  title   = {Classification From Pairwise Similarities/Dissimilarities and Unlabeled Data via Empirical Risk Minimization},
  journal = {Neural Computation},
  volume  = {33},
  number  = {5},
  pages   = {1234--1268},
  year    = {2021},
  doi     = {10.1162/neco_a_01373}
}

@inproceedings{feng2021,
  author    = {Feng, Lei and Shu, Senlin and Lu, Nan and Han, Bo and Xu, Miao and Niu, Gang and An, Bo and Sugiyama, Masashi},
  title     = {Pointwise Binary Classification with Pairwise Confidence Comparisons},
  booktitle = {Proceedings of the 38th International Conference on Machine Learning (ICML 2021)},
  year      = {2021}
}

@book{sugiyama2022,
  author    = {Sugiyama, Masashi and Bao, Han and Ishida, Takashi and Lu, Nan and Sakai, Tomoya and Niu, Gang},
  title     = {Machine Learning from Weak Supervision: An Empirical Risk Minimization Approach},
  publisher = {The MIT Press},
  year      = {2022},
  isbn      = {9780262047074}
}

@inproceedings{cao2021,
  author    = {Cao, Yuzhou and Feng, Lei and Xu, Yitian and An, Bo and Niu, Gang and Sugiyama, Masashi},
  title     = {Learning from Similarity-Confidence Data},
  booktitle = {Proceedings of the 38th International Conference on Machine Learning (ICML 2021)},
  year      = {2021}
}

@inproceedings{wang2023,
  author    = {Wang, Wei and Feng, Lei and Jiang, Yuchen and Niu, Gang and Zhang, Min-Ling and Sugiyama, Masashi},
  title     = {Binary Classification with Confidence Difference},
  booktitle = {Proceedings of the 37th Annual Conference on Neural Information Processing Systems (NeurIPS 2023)},
  year      = {2023},
}

@article{LeCun1998MNIST,
  author  = {LeCun, Yann and Bottou, L{\'e}on and Bengio, Yoshua and Haffner, Patrick},
  title   = {Gradient-Based Learning Applied to Document Recognition},
  journal = {Proceedings of the IEEE},
  volume  = {86},
  number  = {11},
  pages   = {2278--2324},
  year    = {1998},
}

@article{Clanuwat2018KMNIST,
  author  = {Clanuwat, Tarin and Bober-Irizar, Mikel and Kitamoto, Asanobu and Lamb, Alex and Yamamoto, Kazuaki and Ha, David},
  title   = {Deep Learning for Classical Japanese Literature},
  journal = {arXiv preprint arXiv:1812.01718},
  year    = {2018}
}

@article{Xiao2017FashionMNIST,
  author  = {Xiao, Han and Rasul, Kashif and Vollgraf, Roland},
  title   = {{Fashion-MNIST}: a Novel Image Dataset for Benchmarking Machine Learning Algorithms},
  journal = {arXiv preprint arXiv:1708.07747},
  year    = {2017}
}

@techreport{Krizhevsky2009CIFAR10,
  author      = {Krizhevsky, Alex},
  title       = {Learning Multiple Layers of Features from Tiny Images},
  institution = {University of Toronto},
  year        = {2009}
}

@misc{Alpaydin1998Optdigits,
  author       = {Alpaydin, E. and Kaynak, C.},
  title        = {Optical Recognition of Handwritten Digits},
  howpublished = {UCI Machine Learning Repository},
  year         = {1998},
  doi          = {10.24432/C50P49},
  note         = {Dataset}
}

@misc{Alpaydin1996Pendigits,
  author       = {Alpaydin, E. and Alimoglu, Fevzi},
  title        = {Pen-Based Recognition of Handwritten Digits},
  howpublished = {UCI Machine Learning Repository},
  year         = {1996},
  doi          = {10.24432/C5MG6K},
  note         = {Dataset}
}

@misc{Slate1991Letter,
  author       = {Slate, David},
  title        = {Letter Recognition},
  howpublished = {UCI Machine Learning Repository},
  year         = {1991},
  doi          = {10.24432/C5ZP40},
  note         = {Dataset}
}

@misc{Latif2018PMUUD,
  author       = {Latif, Ghazanfar},
  title        = {{PMU-UD}},
  howpublished = {UCI Machine Learning Repository},
  year         = {2018},
  doi          = {10.24432/C5XG84},
  note         = {Dataset}
}

@inproceedings{He2016ResNet,
  author    = {He, Kaiming and Zhang, Xiangyu and Ren, Shaoqing and Sun, Jian},
  title     = {Deep Residual Learning for Image Recognition},
  booktitle = {Proceedings of the IEEE Conference on Computer Vision and Pattern Recognition (CVPR 2016)},
  year      = {2016},
  doi       = {10.1109/CVPR.2016.90}
}

@inproceedings{Nair2010ReLU,
  author    = {Nair, Vinod and Hinton, Geoffrey E.},
  title     = {Rectified Linear Units Improve Restricted Boltzmann Machines},
  booktitle = {Proceedings of the 27th International Conference on Machine Learning (ICML 2010)},
  year      = {2010}
}

@inproceedings{Ioffe2015BatchNorm,
  author    = {Ioffe, Sergey and Szegedy, Christian},
  title     = {Batch Normalization: Accelerating Deep Network Training by Reducing Internal Covariate Shift},
  booktitle = {Proceedings of the 32nd International Conference on Machine Learning (ICML 2015)},
  year      = {2015}
}

@inproceedings{duplessis2015convex,
  title     = {Convex Formulation for Learning from Positive and Unlabeled Data},
  author    = {Du Plessis, Marthinus and Niu, Gang and Sugiyama, Masashi},
  booktitle = {Proceedings of the 32nd International Conference on Machine Learning (ICML 2015)},
  year      = {2015},
}

@inproceedings{lu2019minimal,
  title     = {On the Minimal Supervision for Training Any Binary Classifier from Only Unlabeled Data},
  author    = {Lu, Nan and Niu, Gang and Menon, Aditya Krishna and Sugiyama, Masashi},
  booktitle = {Proceedings of the 7th International Conference on Learning Representations (ICLR 2019)},
  year      = {2019}
}

@article{cour2011learning,
  title   = {Learning from Partial Labels},
  author  = {Cour, Timothee and Sapp, Ben and Taskar, Ben},
  journal = {Journal of Machine Learning Research},
  volume  = {12},
  pages   = {1501--1536},
  year    = {2011}
}

@inproceedings{ishida2017learning,
  title     = {Learning from Complementary Labels},
  author    = {Ishida, Takashi and Niu, Gang and Hu, Weihua and Sugiyama, Masashi},
  booktitle = {Proceedings of the 31st Annual Conference on Neural Information Processing Systems (NIPS 2017)},
  year      = {2017}
}

\appendices

\section{Proof of Theorem \ref{thm:sdpc-weights}}
\label{app:sdpc-construction}

Fix $t\in\{+,-\}$.
Substituting Eq.~\eqref{eq:sdpc-weight} into Eq.~\eqref{eq:sdpc-compatibility} gives
\[
\begin{pmatrix}
\pi_+ & -\pi_-\\
-\pi_- & \pi_+
\end{pmatrix}
\begin{pmatrix}
c_{\mathrm{S},t}\\
c_{\mathrm{D},t}
\end{pmatrix}
=
\Delta\boldsymbol{b}_t.
\]
The determinant of the coefficient matrix on the left-hand side is
\[
\pi_+^2-\pi_-^2
=
(\pi_+-\pi_-)(\pi_++\pi_-)
=
\Delta.
\]
Since $\Delta\neq0$ by assumption, this system has a unique solution:
\[
\begin{pmatrix}
c_{\mathrm{S},t}\\
c_{\mathrm{D},t}
\end{pmatrix}
=
\begin{pmatrix}
\pi_+ & \pi_-\\
\pi_- & \pi_+
\end{pmatrix}
\boldsymbol{b}_t.
\]

For the first element, we substitute
\[
\boldsymbol{b}_+
=
\begin{pmatrix}
1\\
-\pi_+
\end{pmatrix},
\]
which gives
\[
\begin{pmatrix}
c_{\mathrm{S},+}\\
c_{\mathrm{D},+}
\end{pmatrix}
=
\begin{pmatrix}
\pi_+^2\\
\pi_- - \pi_+^2
\end{pmatrix}.
\]
Similarly, for the second element, we substitute
\[
\boldsymbol{b}_-
=
\begin{pmatrix}
-\pi_-\\
1
\end{pmatrix},
\]
which gives
\[
\begin{pmatrix}
c_{\mathrm{S},-}\\
c_{\mathrm{D},-}
\end{pmatrix}
=
\begin{pmatrix}
\pi_-^2\\
\pi_+ - \pi_-^2
\end{pmatrix}.
\]
Substitution of these coefficients into Eq.~\eqref{eq:sdpc-weight} yields
Eqs.~\eqref{eq:a-s-plus}--\eqref{eq:a-d-minus}.

From Eqs.~\eqref{eq:a-s-plus} and \eqref{eq:a-s-minus},
\begin{align*}
\boldsymbol{a}_{\mathrm{S},+}
+
\boldsymbol{a}_{\mathrm{S},-}
&=
\frac{\pi_+^2+\pi_-^2}{\Delta}
\begin{pmatrix}
\pi_+\\
-\pi_-
\end{pmatrix}\\
&=
\frac{\pi_{\mathrm{S}}}{\Delta}
\begin{pmatrix}
\pi_+\\
-\pi_-
\end{pmatrix}
=
2\boldsymbol{a}_{\mathrm{S}}.
\end{align*}
Similarly, Eqs.~\eqref{eq:a-d-plus} and \eqref{eq:a-d-minus} give
\begin{align*}
\boldsymbol{a}_{\mathrm{D},+}
+
\boldsymbol{a}_{\mathrm{D},-}
&=
\frac{
\pi_- - \pi_+^2
+
\pi_+ - \pi_-^2
}{\Delta}
\begin{pmatrix}
-\pi_-\\
\pi_+
\end{pmatrix}\\
&=
\frac{\pi_{\mathrm{D}}}{\Delta}
\begin{pmatrix}
-\pi_-\\
\pi_+
\end{pmatrix}
=
2\boldsymbol{a}_{\mathrm{D}}.
\end{align*}
Therefore, Eq.~\eqref{eq:sd-recovery} holds.

\section{Proof of Theorem \ref{thm:sdpc-sd-decomposition}}
\label{app:proof-sdpc-sd-decomposition}

From Eq.~\eqref{eq:sd-recovery}, $\boldsymbol{a}_{\mathrm{S}}$ is the average of $\boldsymbol{a}_{\mathrm{S},+}$ and $\boldsymbol{a}_{\mathrm{S},-}$. From Eqs.~\eqref{eq:a-s-plus} and \eqref{eq:a-s-minus},
\begin{align*}
\boldsymbol{a}_{\mathrm{S},+}
-
\boldsymbol{a}_{\mathrm{S},-}
=
\frac{\pi_+^2-\pi_-^2}{\Delta}
\begin{pmatrix}
\pi_+\\
-\pi_-
\end{pmatrix}
=
\begin{pmatrix}
\pi_+\\
-\pi_-
\end{pmatrix}.
\end{align*}
Therefore,
\[
\boldsymbol{a}_{\mathrm{S},+}
-
\boldsymbol{a}_{\mathrm{S}}
=
\frac12
\begin{pmatrix}
\pi_+\\
-\pi_-
\end{pmatrix},
\quad
\boldsymbol{a}_{\mathrm{S},-}
-
\boldsymbol{a}_{\mathrm{S}}
=
-\frac12
\begin{pmatrix}
\pi_+\\
-\pi_-
\end{pmatrix}.
\]

Similarly, from Eqs.~\eqref{eq:a-d-plus} and \eqref{eq:a-d-minus},
\begin{align*}
\boldsymbol{a}_{\mathrm{D},+}
-
\boldsymbol{a}_{\mathrm{D},-}
&=
\frac{
(\pi_- - \pi_+^2)
-
(\pi_+ - \pi_-^2)
}{\Delta}
\begin{pmatrix}
-\pi_-\\
\pi_+
\end{pmatrix}\\
&=
-2
\begin{pmatrix}
-\pi_-\\
\pi_+
\end{pmatrix}.
\end{align*}
From Eq.~\eqref{eq:sd-recovery},
$\boldsymbol{a}_{\mathrm{D}}$ is the average of
$\boldsymbol{a}_{\mathrm{D},+}$ and
$\boldsymbol{a}_{\mathrm{D},-}$.
Therefore,
\[
\boldsymbol{a}_{\mathrm{D},+}
-
\boldsymbol{a}_{\mathrm{D}}
=
-
\begin{pmatrix}
-\pi_-\\
\pi_+
\end{pmatrix},
\quad
\boldsymbol{a}_{\mathrm{D},-}
-
\boldsymbol{a}_{\mathrm{D}}
=
\begin{pmatrix}
-\pi_-\\
\pi_+
\end{pmatrix}.
\]

Subtracting Eq.~\eqref{eq:sd-estimator} from Eq.~\eqref{eq:sdpc-objective} gives
\begin{align*}
&\widehat J_{\mathrm{SDPC}}(g)
-
\widehat R_{\mathrm{SD}}(g)\\
&=
\left(
\boldsymbol{a}_{\mathrm{S},+}
-
\boldsymbol{a}_{\mathrm{S}}
\right)^{\top}
\widehat{\boldsymbol{\phi}}_{\mathrm{S},+}(g)
+
\left(
\boldsymbol{a}_{\mathrm{S},-}
-
\boldsymbol{a}_{\mathrm{S}}
\right)^{\top}
\widehat{\boldsymbol{\phi}}_{\mathrm{S},-}(g)\\
&\quad+
\left(
\boldsymbol{a}_{\mathrm{D},+}
-
\boldsymbol{a}_{\mathrm{D}}
\right)^{\top}
\widehat{\boldsymbol{\phi}}_{\mathrm{D},+}(g)
+
\left(
\boldsymbol{a}_{\mathrm{D},-}
-
\boldsymbol{a}_{\mathrm{D}}
\right)^{\top}
\widehat{\boldsymbol{\phi}}_{\mathrm{D},-}(g)\\
&=
\frac12
\begin{pmatrix}
\pi_+\\
-\pi_-
\end{pmatrix}^{\top}
\left\{
\widehat{\boldsymbol{\phi}}_{\mathrm{S},+}(g)
-
\widehat{\boldsymbol{\phi}}_{\mathrm{S},-}(g)
\right\}\\
&\quad-
\begin{pmatrix}
-\pi_-\\
\pi_+
\end{pmatrix}^{\top}
\left\{
\widehat{\boldsymbol{\phi}}_{\mathrm{D},+}(g)
-
\widehat{\boldsymbol{\phi}}_{\mathrm{D},-}(g)
\right\}.
\end{align*}
The right-hand side is equal to $\widehat B_{\mathrm{SDPC}}(g)$ defined in Eq.~\eqref{eq:empirical-order-term}.
Therefore,
\[
\widehat J_{\mathrm{SDPC}}(g)
=
\widehat R_{\mathrm{SD}}(g)
+
\widehat B_{\mathrm{SDPC}}(g)
\]

\section{Proof of Theorem \ref{thm:sdpc-pcomp-decomposition}}
\label{app:proof-sdpc-pcomp-decomposition}

For $t\in\{+,-\}$, Eqs.~\eqref{eq:sdpc-compatibility} and \eqref{eq:ct-definition} give
\begin{align*}
\boldsymbol{a}_{\mathrm{S},t}
-
w_{\mathrm{S}}\boldsymbol{b}_t
=
w_{\mathrm{D}}\boldsymbol{a}_{\mathrm{S},t}
-
w_{\mathrm{S}}\boldsymbol{a}_{\mathrm{D},t}
=
\boldsymbol{c}_t,
\end{align*}
and
\begin{align*}
\boldsymbol{a}_{\mathrm{D},t}
-
w_{\mathrm{D}}\boldsymbol{b}_t
=
w_{\mathrm{S}}\boldsymbol{a}_{\mathrm{D},t}
-
w_{\mathrm{D}}\boldsymbol{a}_{\mathrm{S},t}
=
-\boldsymbol{c}_t.
\end{align*}

Therefore,
\[
\boldsymbol{a}_{\mathrm{S},t}
=
w_{\mathrm{S}}\boldsymbol{b}_t
+
\boldsymbol{c}_t,
\quad
\boldsymbol{a}_{\mathrm{D},t}
=
w_{\mathrm{D}}\boldsymbol{b}_t
-
\boldsymbol{c}_t.
\]

Substituting these expressions into Eq.~\eqref{eq:sdpc-objective} gives
\begin{align*}
\widehat J_{\mathrm{SDPC}}(g)
&=
w_{\mathrm{S}}
\Bigl[
\boldsymbol{b}_+^{\top}
\widehat{\boldsymbol{\phi}}_{\mathrm{S},+}(g)
+
\boldsymbol{b}_-^{\top}
\widehat{\boldsymbol{\phi}}_{\mathrm{S},-}(g)
\Bigr]\\
&\quad+
w_{\mathrm{D}}
\Bigl[
\boldsymbol{b}_+^{\top}
\widehat{\boldsymbol{\phi}}_{\mathrm{D},+}(g)
+
\boldsymbol{b}_-^{\top}
\widehat{\boldsymbol{\phi}}_{\mathrm{D},-}(g)
\Bigr]\\
&\quad+
\boldsymbol{c}_+^{\top}
\left\{
\widehat{\boldsymbol{\phi}}_{\mathrm{S},+}(g)
-
\widehat{\boldsymbol{\phi}}_{\mathrm{D},+}(g)
\right\}\\
&\quad+
\boldsymbol{c}_-^{\top}
\left\{
\widehat{\boldsymbol{\phi}}_{\mathrm{S},-}(g)
-
\widehat{\boldsymbol{\phi}}_{\mathrm{D},-}(g)
\right\}.
\end{align*}

By Eq.~\eqref{eq:pcomp-estimator-vector}, the first two terms on the right-hand side are equal to
$\widehat R_{\mathrm{PC}}(g)$, and the remaining two terms are equal to
$\widehat C_{\mathrm{SDPC}}(g)$ defined in Eq.~\eqref{eq:empirical-sd-label-term}.
Therefore,
\[
\widehat J_{\mathrm{SDPC}}(g)
=
\widehat R_{\mathrm{PC}}(g)
+
\widehat C_{\mathrm{SDPC}}(g).
\]

\end{document}